\documentclass[11pt]{article}
\usepackage{coling2020}
\usepackage{times}
\usepackage{url}
\usepackage{latexsym}

\usepackage{amsmath}
\usepackage{amssymb}
\usepackage{mathrsfs}
\usepackage{xcolor}
\usepackage{amsthm}
\usepackage{amsfonts}
\usepackage{graphicx}
\usepackage{stmaryrd}
\usepackage{algorithm}
\usepackage{xr}

\usepackage{times}
\usepackage{latexsym}
\usepackage{url}
\usepackage{array}
\usepackage{booktabs}
\usepackage{multirow}
\usepackage{balance}
\usepackage{colortbl}
\usepackage{makecell}
\usepackage{soul}
\usepackage{blindtext}
\usepackage{multirow}
\usepackage{graphicx}
\usepackage{caption}
\usepackage{subcaption}
\newcolumntype{P}[1]{>{\centering\arraybackslash}p{#1}}

\usepackage{microtype}

\theoremstyle{definition}
\newtheorem{definition}{Definition}[section]

\theoremstyle{remark}

\theoremstyle{plain}
\newtheorem{theorem}{Theorem}[section]

\theoremstyle{plain}

\theoremstyle{plain}

\newtheorem{proposition}[theorem]{Proposition}
\newtheorem{lemma}[theorem]{Lemma}

\usepackage{graphicx,wrapfig,lipsum}

\usepackage{amsmath,amsfonts,bm}

\def\eqref#1{(\ref{#1})}

\def\1{\bm{1}}

\def\rmA{{\mathbf{A}}}

\def\rmW{{\mathbf{W}}}

\def\vzero{{\bm{0}}}

\def\vb{{\bm{b}}}
\def\vc{{\bm{c}}}

\def\vh{{\bm{h}}}

\def\vq{{\bm{q}}}

\def\vt{{\bm{t}}}

\def\vx{{\bm{x}}}

\def\vz{{\bm{z}}}

\def\mW{{\bm{W}}}

\DeclareMathAlphabet{\mathsfit}{\encodingdefault}{\sfdefault}{m}{sl}
\SetMathAlphabet{\mathsfit}{bold}{\encodingdefault}{\sfdefault}{bx}{n}

\usepackage{commath}
\usepackage{amssymb}
\usepackage{amsmath,amsfonts,bm}

\newcommand{\bq}{\mathbb{Q}}

\newcommand{\stk}{\bm{\omega}}
\newcommand{\topv}{\bm{\tau}_{t-1}^{top}}
\newcommand{\psh}{\bm{\tau}_{t-1}^{push}}

\newcommand{\vphi}{\bm{\phi}}
\newcommand{\vpsi}{\bm{\psi}}

\newcommand{\sigmoid}{\sigma}

\newcommand{\dyck}[1]{Dyck-$#1$}

\colingfinalcopy 

\title{On the Practical Ability of Recurrent Neural Networks \\ to Recognize Hierarchical Languages}

\author{Satwik Bhattamishra$^\spadesuit$  \quad Kabir Ahuja$^\diamondsuit$\thanks{\, This research was conducted during the author's internship at Microsoft Research.} \quad Navin Goyal$^\spadesuit$\\
	$^\spadesuit$ Microsoft Research India\\
	$^\diamondsuit$ Udaan.com\\
	{\tt \small \{t-satbh,navingo\}@microsoft.com} \\
	{\tt \small kabir.ahuja@udaan.com} \\
}

\date{}

\begin{document}
\maketitle
\begin{abstract}
	
While recurrent models have been effective in NLP tasks, their performance on context-free languages (CFLs) has been found to be quite weak. Given that CFLs are believed to capture important phenomena such as hierarchical structure in natural languages, this discrepancy in performance calls for an explanation. We study the performance of recurrent models on \dyck{n} languages, a particularly important and well-studied class of CFLs. We find that while recurrent models generalize nearly perfectly if the lengths of the training and test strings are from the same range, they perform poorly if the test strings are longer. At the same time, we observe that recurrent models are expressive enough to recognize Dyck words of arbitrary lengths in finite precision if their depths are bounded. Hence, we evaluate our models on samples generated from Dyck languages with bounded depth and find that they are indeed able to generalize to much higher lengths. Since natural language datasets have nested dependencies of bounded depth, this may help explain why they perform well in modeling hierarchical dependencies in natural language data despite prior works indicating poor generalization performance on Dyck languages. We perform probing studies to support our results and provide comparisons with Transformers.

\end{abstract}
\section{Introduction}
\blfootnote{ This work is licensed under a Creative Commons Attribution 4.0 International Licence. Licence details: \url{http://creativecommons.org/licenses/by/4.0/}.}

Recurrent models (RNNs and more specifically LSTMs) have been used extensively across several NLP tasks such as machine translation \cite{sutskever2014sequence}, language modeling \cite{melis2017state} and question answering \cite{seo2016bidirectional}. Natural languages involve phenomena such as hierarchical and long-distance dependencies. Although RNNs are known to be Turing-complete \cite{siegelmann1992computational} given unbounded precision, their practical ability to model such phenomena remains unclear.

Recently, several works \cite{weiss2018practical,sennhauser-berwick-2018-evaluating,skachkova-etal-2018-closing} have attempted to understand the capabilities of LSTMs by empirically analyzing them on different types of formal languages. Natural languages, for the most part, can be modeled by context-free languages \cite{Jaeger-Rogers} and their hierarchical structure has been emphasized by \newcite{chomsky2002syntactic}. Thus studying the capabilities of RNNs in recognizing context-free languages (CFLs) can shed light on how well they can model hierarchical structures. An important family of context-free languages is the \dyck{n} language\footnote{Informally, the Chomsky--Sch{\"u}tzenberger representation theorem \shortcite{chomsky1959algebraic} asserts that \dyck{n} languages for $n \geq 1$ capture the complexity of context-free languages in a precise sense.}.

Previous works \cite{suzgun2019lstm,suzgun2019memory,yu-etal-2019-learning} showed that LSTMs achieve limited generalization performance on recognizing \dyck{2}. This prompted the development of memory-augmented variants \cite{joulin2015inferring,suzgun2019memory} of LSTMs which generalize well on Dyck languages but are notoriously hard to train and fail to perform well on practical NLP tasks \cite{NIPS2019_8748}. On the other hand, despite the limited performance of LSTMs on Dyck languages, several studies \cite{gulordava-etal-2018-colorless,tran-etal-2018-importance} have found that LSTMs perform well in modeling hierarchical structure in natural language data. In this work, we take a step towards bridging this gap.



\textbf{Our Contributions.} We investigate the ability of recurrent models to learn and generalize on Dyck languages. We first evaluate the ability of LSTMs to recognize randomly sampled \dyck{n} sequences and find, in contrast to prior works \cite{suzgun2019lstm,suzgun2019memory}, that they can generalize nearly perfectly when the test samples are within the same lengths as seen during training. Similar to prior works, when evaluated on randomly generated samples of higher lengths we observe limited performance. Dyck languages and (deterministic) CFLs can be recognized by (deterministic) pushdown automata (PDA). We construct an RNN that directly simulates a PDA given unbounded precision. A key observation is that the higher the depth of the stack the higher is the required precision. This implies that fixed precision RNNs are expressive enough to recognize strings of arbitrary lengths if the required depth of the stack is bounded. Based on this observation, we test the hypothesis whether LSTMs can generalize to higher lengths if the depth of the inputs in the training and test set is bounded by the same value. In the bounded depth setting, LSTMs are able to generalize to much higher lengths compared to the lengths used during training. Given that natural languages in practical settings also contain nested dependencies of bounded depths \cite{gibson1991computational,mcelree2001working}, this may help explain why LSTMs perform well in modeling natural language corpora containing nested dependencies. We then assess the generalization performance of the model across higher depths and find that although LSTMs can generalize to a certain extent, their performance degrades gradually with increasing depths. Since Transformer  \cite{vaswani2017attention}  is also a dominant model in NLP \cite{devlin-etal-2019-bert}, we include it in our experiments. To our knowledge, prior works have not empirically analyzed Transformers on formal languages, particularly context-free languages. We further conduct robustness experiments and probing studies to support our results.

\begin{figure}[t]
	\centering
	\includegraphics[width=.65\textwidth]{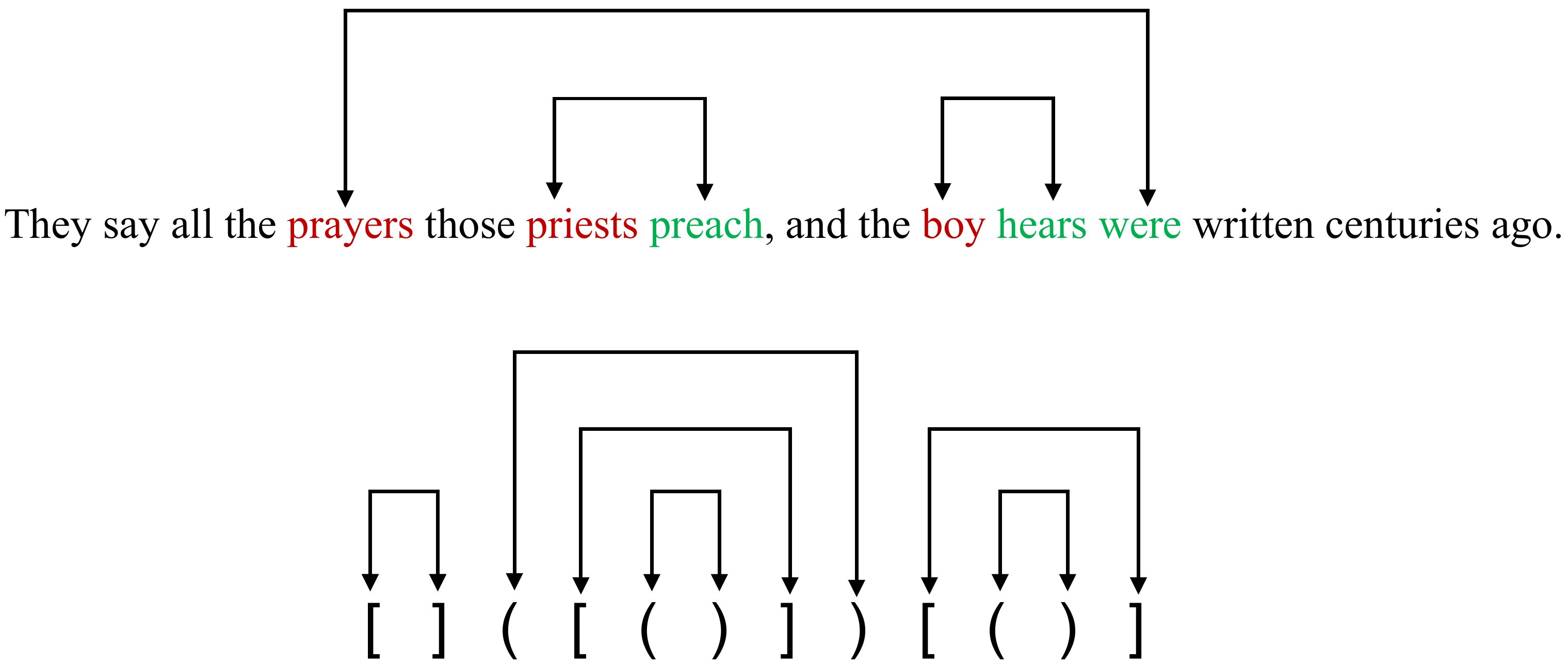}
	\caption{\label{fig:intro_fig} Nested dependencies in English sentences and in Dyck-2.}
\end{figure}

\section{Preliminaries and Motivation}\label{sec:def}

The language \dyck{n} parameterized by $n \geq 1$ consists of well-formed words from $n$ types of parentheses. Its derivation rules are: $S \rightarrow (_iS)_i, \; S \rightarrow SS, \; S \rightarrow \epsilon$ where $i \in \{1, \ldots, n\}$.  \dyck{n} is context-free for every $n$, \dyck{2} being crucial among them, since all \dyck{n} for $n>2$ can be reduced to \dyck{2} \cite{suzgun2019lstm}. 
For words in \dyck{n}, the required depth of the stack in its underlying PDA is the maximum number of unbalanced parentheses in a prefix. For instance, the word ( [ ( ) ] ) [ ]  in \dyck{2} has maximum depth of 3 corrsponding to the prefix ( [ (.

\textbf{RNNs.} RNNs are defined by the update rule $\vh_t = f(\vh_{t-1}, \vx_t)$, where the function $f$ could be a feedforward network, $\vh_t$ is the model's memory vector usually referred to as the hidden state vector and $\vx_t$ denotes the input vector at the $t$-th step. In practice, $f$ is usually a single layer feedforward network with $\textsf{tanh}$ or $\textsf{ReLU}$ activation. To mitigate the vanishing gradients problem, LSTMs \cite{hochreiter1997long}, a variant of RNNs with additional gates, is most commonly used in practice. In our experiments, we will primarily work with LSTMs. 

\subsection{Expressiveness Results}

\begin{proposition}\label{prop:pda}
	Any Deterministic Pushdown Automaton can be simulated by an RNN with $\textsf{ReLU}$ activation.
\end{proposition}

We provide a proof by construction for the above result in Appendix~\ref{sec:construction} by using the Cantor-set like encoding scheme as introduced in Siegelmann and Sontag \shortcite{siegelmann1992computational} to emulate stack operations. The above result was first obtained by \newcite{korsky2019computational} somewhat indirectly using the \newcite{chomsky1959algebraic} theorem. The above Proposition implies that RNNs can recognize any deterministic CFL given unbounded precision. 
In the construction, the higher the required depth of the stack the higher the required precision. This implies that fixed precision RNNs are expressive enough to recognize strings of arbitrary lengths if the required depth of the stack is bounded. This can also be gleaned from the construction of \newcite{korsky2019computational} with some additional work\footnote{Since our construction, although novel, is not critical for the inferences, it has been moved to the appendix.}.

\section{Experiments}\label{sec:exp}


\begin{table*}[t]
	\scriptsize{\centering
		\begin{tabular}{P{6em}p{7em}P{8em}P{8em}P{8em}P{8em}P{8em}}
			\toprule
			\textbf{Language} & \textbf{Model} &
			\multicolumn{2}{c}{\textbf{Vanilla (Randomly Sampled)}}&
			\multicolumn{3}{c}{\textbf{Bounded Depth}}\\
			\cmidrule(lr){3-4} \cmidrule(lr){5-7}
			&&
			\textbf{Bin-1A Accuracy [2, 50]}$\uparrow$ & \textbf{Bin-2A Accuracy [52, 100]}$\uparrow$ &
			\textbf{Bin-1B Accuracy [2, 50]}$\uparrow$ &
			\textbf{Bin-2B Accuracy [52, 100]}$\uparrow$ &
			\textbf{Bin-3B Accuracy [102, 150]}$\uparrow$\\
			\midrule
			\multirow{2}{*}{\textbf{\dyck{2}}} & \textbf{LSTM} & \textbf{99.5}  & \textbf{75.1} & \textbf{99.9} & \textbf{99.6} & \textbf{96.0}\\
			&\textbf{Transformer} & 95.1 & 21.8 & \textbf{99.9} & 92.1 & 36.3\\
			\midrule
			\multirow{2}{*}{\textbf{\dyck{3}}}&\textbf{LSTM} & \textbf{97.3}  & \textbf{54.0} & \textbf{99.7} & \textbf{96.3} & \textbf{89.5}\\
			&\textbf{Transformer} & 87.7  & 26.2 & 90.1 & 48.9 & 6.4\\
			\midrule
			\multirow{2}{*}{\textbf{\dyck{4}}}&\textbf{LSTM} & \textbf{97.8}  & \textbf{50.7} & \textbf{99.9} &\textbf{ 95.1} & \textbf{87.0}\\
			&\textbf{Transformer} & 92.7  & 36.6 & 94.4 & 49.3 & 5.6\\
			\bottomrule
		\end{tabular}
		\caption{\label{tab:results} The performance of neural models on considered Dyck languages. The reported scores are obtained by averaging the accuracies of 5 best performing hyperparameter configurations of each model. All models are trained on inputs with length in [2,50] and evaluated on validation sets. Bin-2A contains higher lengths without any restriction on depth. In Bin-1B, Bin-2B and Bin-3B, the test inputs have their depths upper bounded by $10$.}
	}
\end{table*}

\textbf{Setup.} In our experiments, we consider three Languages, namely \dyck{2}, \dyck{3} and \dyck{4}. For each of the three languages, we generate 3 different types of training and validation sets. In the first case, we generate 10k strings for the training set within lengths [2, 50] and generate two validation sets each containing 1k strings within lengths [2, 50] and [52, 100] respectively without any restriction on the depth of the Dyck words. Our second setting also resembles the previous dataset in terms of lengths of the strings in the training and validation sets. However, in this case, we restrict all the strings to have depths in the range [1, 10]. This is to test the generalization ability across lengths when the depth is bounded. In the third case, we test the generalization ability across depths, when the lengths in train and validation sets are in the same interval [2,100]. Along with LSTMs, we also report the performance of Transformers (as used in GPT \cite{radford2018improving}) on each task since it is also a dominant model in NLP

\textbf{Tasks.} We train and evaluate our models on the Next Character Prediction Task (NCP) introduced in \newcite{gers2001lstm} and used in \newcite{suzgun2019lstm} and \newcite{suzgun2019memory}. Similar to an LM setup, the model is only presented with positive samples from a given language. In NCP, for a sequence of symbols  $s_1, s_2, \ldots, s_n$, the model is presented with the sequence  $s_1, s_2, \ldots, s_i$ at the $i^{th}$ step and the goal of the model is to predict the set of valid characters for the $(i+1)^{th}$ step, which can be represented as a $k$-hot vector. The model assigns a probability to each character in the vocabulary corresponding to its validity in the next step, which is achieved by applying sigmoid activation over the unnormalized scores predicted through its output layer. Following \newcite{suzgun2018evaluating} and \newcite{suzgun2019lstm}, we use mean squared error between the predicted probabilities and $k$-hot labels as the loss function. During inference, we use a threshold of 0.5 to obtain the final prediction. The model’s prediction is considered to be correct if and only if its output at every step is correct. The accuracy of the model over test samples is the fraction of total samples which are predicted correctly. Note that, this is a relatively stringent metric as a correct prediction is obtained only when the model's output is correct at every step as opposed to standard classification tasks. Refer to \newcite{bhattamishra2020ability} for a discussion on the choice of character prediction task and its relation with other tasks such as standard classification and language modeling. Details of the dataset and parameters relevant for reproducibility can be found in section \ref{sec:exp_details} in Appendix. We have made our source code available at \url{https://github.com/satwik77/RNNs-Context-Free}.

\subsection{Results}

Table \ref{tab:results} shows the performance of LSTMs and Transformers on the considered languages. When inputs are randomly sampled in a given range of lengths, LSTMs can generalize well on the validation set containing inputs of the same lengths as seen during training (Bin-1A)\footnote{This result is in disagreement with the results reported in \newcite{suzgun2019lstm} and \newcite{suzgun2019memory}.  The possible discrepancy could be due to more extensive hyperparameter tuning in our experiments. We found this holds even while tuning within the same constraints as mentioned in their setup.}, for all considered Dycks. However, for higher lengths (Bin-2A), it struggles to generalize on these languages. For the case when the depth is bounded, LSTMs generalize very well to much higher lengths (Bin 1B, 2B, and 3B). Transformers, on the other hand, fail to generalize to higher lengths in either case, which might be attributed to the fact that at test time, it receives positional encodings that it was not trained on. To investigate the generalization ability of models across increasing depths, we trained the models up to depth $15$ and evaluated on $5$ validation sets with an incremental increase in depth in each set (refer to Figure \ref{fig:depth_gen}). We found that the models were able to generalize up to a certain extent but their performance degraded gradually as we increase the depth. However, Transformers performed relatively better as compared to LSTMs.


\begin{figure}[t]
		\centering
		\includegraphics[width=.6\textwidth]{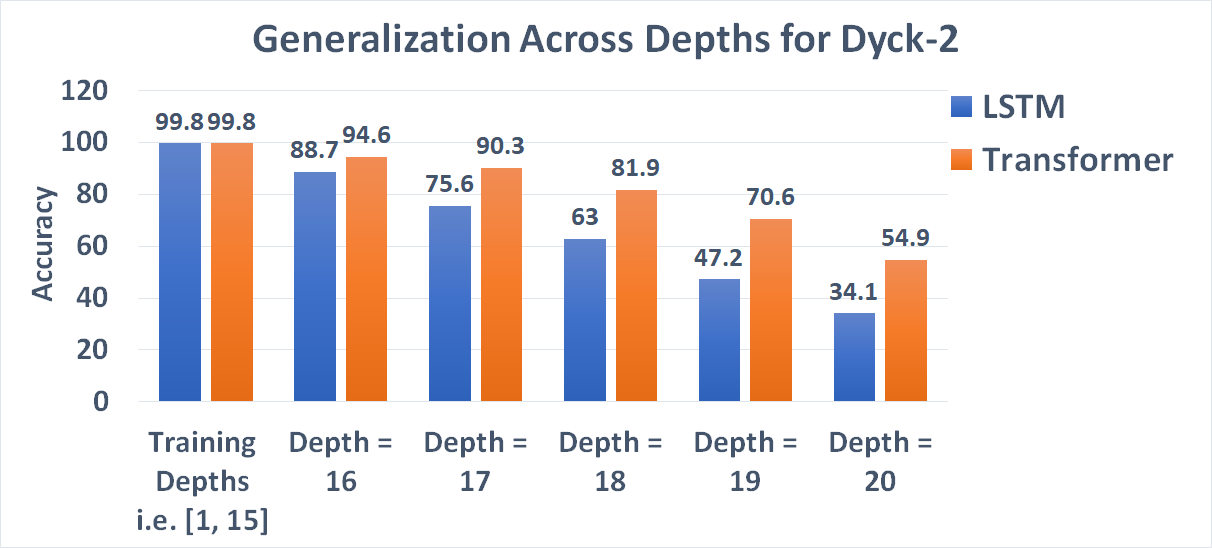}
		\caption{\label{fig:depth_gen} Generalization of LSTMs and Transformers on higher depths. The lengths of strings in the training set and all validation sets were fixed to lie between 2 to 100.}
\end{figure}

\noindent \textbf{Probing.} We also conducted probing experiments on LSTMs to better understand how they perform these particular tasks. We first attempt to extract the depth of the underlying stack from the cell state of an LSTM model trained to recognize \dyck{2}. We found that a single layer feedforward network was easily able to extract the depth with perfect ($100\%$) accuracy and generalize to unseen data. Figure \ref{fig:tsne_vis} shows a visualization of $t$-SNE projection of the hidden state labeled by their corresponding depths. Additionally, we also try to extract the stack elements from the hidden states of the network. For \dyck{2} samples within lengths [2,50] and depths [1,10], along with training LSTM for the NCP task, we co-train auxiliary classifiers to predict each element of the stack up to depth $10$, i.e. the hidden state of the LSTM that is used to predict the next set of valid characters is now also utilized in parallel to predict (by supplying 10 separate linear layers for each element) the elements of the stack.  We find that not only was the model able to predict the elements in a validation set from the same distribution, it was also able to extract the elements for sequences of higher lengths ([52,150]) on which it was not trained on (see Figure \ref{fig:stack_ext}). This further provides evidence to show that LSTMs are able to robustly generalize to inputs of higher lengths when their depths are bounded. We also conducted a few additional robustness experiments to ensure that the model does not overfit on training distribution. Details of probing tasks as well as additional robustness experiments can be found in the Appendix.

\begin{figure}[h]
	\centering
	\begin{subfigure}[b]{0.5\textwidth}
		\centering
		\includegraphics[width=.9\textwidth]{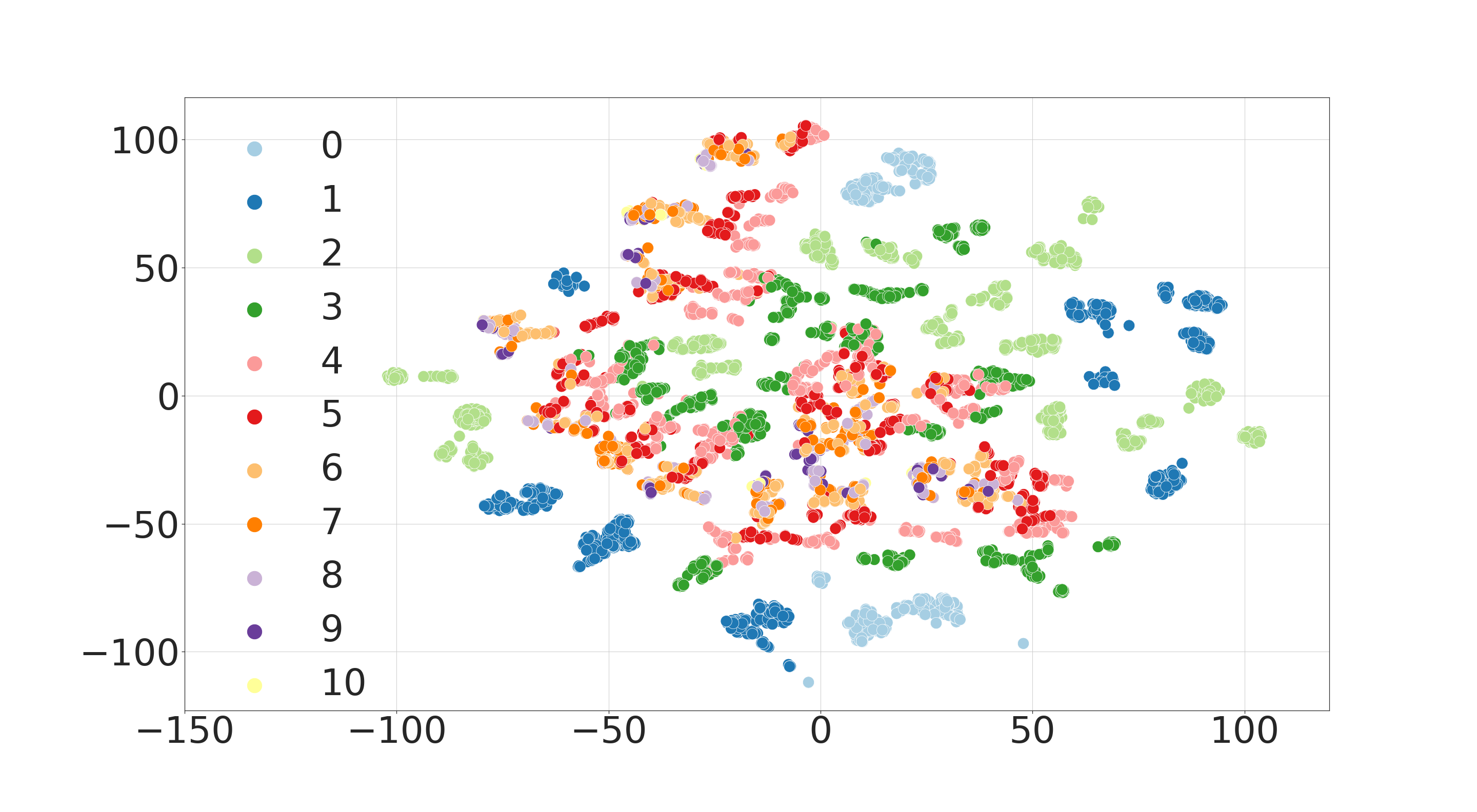}
		\caption{\label{fig:tsne_vis} Visualization of $t$-SNE Projections of the hidden states obtained from a pre-trained LSTM for different \dyck{2} substrings, colored by their depths.}
	\end{subfigure}%
	\begin{subfigure}[b]{0.5\textwidth}
		\centering
		\includegraphics[width=.9\textwidth]{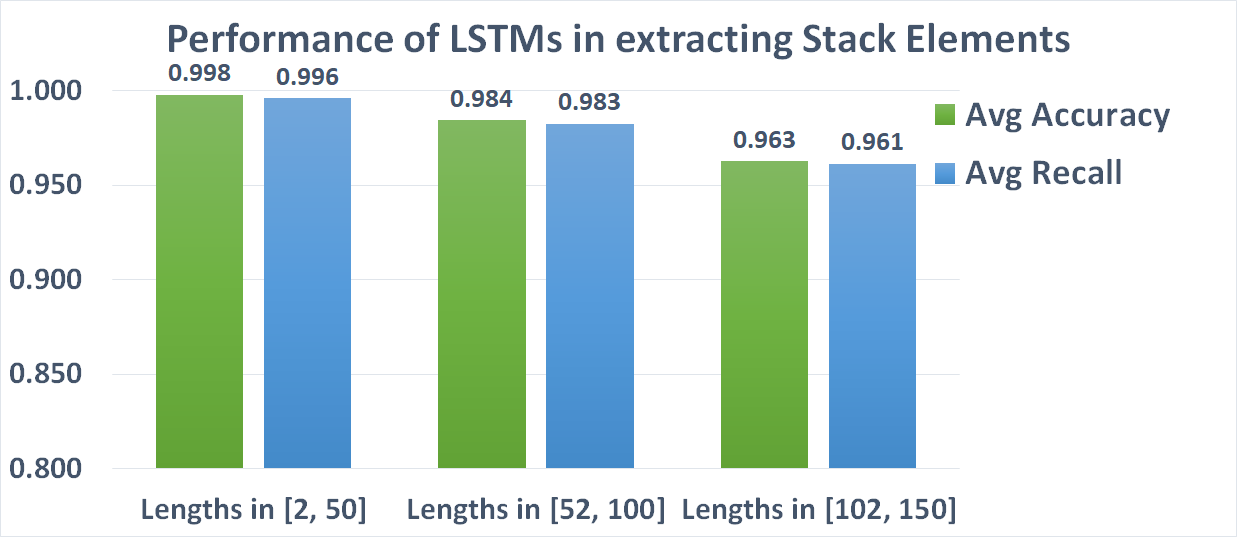}
		\caption{\label{fig:stack_ext} Accuracies and recalls obtained on extracting top-10 elements (averaged over them) of the stack corresponding to different \dyck{2} strings using the hidden state vector of the LSTM.}
	\end{subfigure}
	\caption{}
\end{figure}

\section{Discussion}\label{sec:res}

LSTMs and Transformers have been effective on language modeling tasks. In practice, during language modeling on a natural language corpus, the entire input is fed sequentially to the LSTM. Hence, the length of the input processed by the LSTM is bound to be large, requiring it to model a number of nested dependencies. Prior works in psycholinguistics \cite{gibson1991computational,mcelree2001working} have pointed out that given limited working memory of humans, natural language as used in practice should have nested dependencies of bounded depth. Some works \cite{jin2018unsupervised,noji2016using} have even sought to build parsers for natural language with depth-bounded PCFGs. Our generalization results for LSTMs on depth-bounded CFGs could help explain why they perform well on modeling hierarchical dependencies in natural language datasets. For Transformer, although it did not generalize to higher lengths, but in practice Transformers (as used in BERT and GPT) process inputs in a fixed-length context window. Our results indicate that it does not have trouble in generalizing when the train and validation sets contain inputs of the same lengths.

Our experiments also demonstrate that the limiting factor in LSTMs is in generalizing to higher depths as opposed to its memory-augmented variants. The exact mechanism with which trained LSTMs perform the task is not entirely clear. The limited performance of LSTMs could be due to precision issues or unstable stack encoding mechanism \cite{stogin2020provably}. However, given that natural language datasets are likely to have nested dependencies of bounded depth, this limitation may not play a significant role and may help explain why prior works \cite{gulordava-etal-2018-colorless,tran-etal-2018-importance} have found LSTMs to perform well in modeling hierarchical structure on natural language datasets.

\section*{Acknowledgements}
We thank the anonymous reviewers for their constructive comments and suggestions. We would also like to thank our colleagues at Microsoft Research and Michael Hahn for their valuable feedback and helpful discussions.

\bibliographystyle{coling}
\bibliography{citations}

\clearpage
\appendix

\section{Preliminaries}\label{sec:definitions}

\begin{definition}[Deterministic Pushdown Automata \cite{hopcroft2001introduction}]
	A DPDA is a $7$-tuple $\langle \Sigma, Q, \Gamma, q_0, Z_0,\delta, F \rangle$ with
	\begin{enumerate}
		\item A finite alphabet $\Sigma$
		\item A finite set of states $Q$
		\item A finite stack alphabet
		\item An initial state $q_0$
		\item An initial stack symbol $Z_0$
		\item $F \subseteq Q$ set of accepting states

		\item A state transition function
		\begin{equation*}
		\delta : \Sigma \times Q \times \Gamma \rightarrow (q, \gamma)
		\end{equation*}
		The output of $\delta$ is a finite set of pairs $(q, \gamma)$, where $q$ is a new state and $\gamma$ is the string of stack symbols that replaces the top of the stack. If $X$ is at the top of the stack, then $\gamma= \epsilon$ indicates that the stack is popped, $\gamma=X$ denotes that the stack is unchanged and if $\gamma = YZ$, then $X$ is replaced by $Z$, and $Y$ is pushed onto the stack.

	\end{enumerate}
\end{definition}

A machine processes an input string $x \in \Sigma^*$ one token at a time. For each input token, the machine looks at the current input, state and top of the stack to make a transition into a new state and update its stack. The machine can also take empty string as input and make a transition based on the stack and its current state. The machine halts after reading the input and a string is accepted if the state after reading the complete input is a final state.

	\subsection{Cantor-set Encodings} \label{subsec:cantor}
	
	In our construction, we will make use of Cantor-set like encodings as introduced in \cite{siegelmann1992computational}. The Cantor-set like encodings in base-$4$ provides us a means to encode a stack of values $0$s and $1$s and easily apply stack operations like push, pop and top to update them. Let $\upsilon_s$ denote the encoding of a stack. The contents of the stack can be viewed as a rational number $\frac{p}{4^q} $ where $0 < p < 4^q$. The $i$-th element from the top of the stack can be seen as the $i$-th element to the right of the decimal point the in a base-$4$ expansion. A $0$ stored in the stacked is associated with a $1$ in the expansion while a $1$ stored in the stack is associated with $3$. Hence, only numbers of the special form $\Sigma_{i=1}^{t} \frac{a_i}{4^i}$ where $a_i \in \{1,3\}$ will appear in the encoding. For inputs of the form $I \in \{0,1\}$, the standard stack operations can be applied by simple affine operations. For instance, push($I$) operation can be obtained by $\upsilon_s \mapsto \frac{1}{4}\upsilon_s + \frac{1}{2}I + \frac{1}{4}$ and the pop($I$) can be obtained by $\upsilon_s \mapsto 4\upsilon_s - 2I -1$. The top of the stack can be obtained by $\sigma(4\upsilon_s -2)$ which will be $1$ if the top of the stack is $1$ else $0$. The emptiness of a stack can be checked by $\sigma(\upsilon_s)$ which will be $1$ if the stack is nonempty or else it will be $0$.

\section{Construction}\label{sec:construction}

We will make use of some intermediate notions to describe our construction. We will use these multiple times in our construction. Particularly, Lemma \ref{lem:map_pair} will be used to combine the information of the state vector, input and the symbol at the top of the stack. Lemma \ref{lem:map_trans} and Lemma \ref{lem:map_cont} will be used to implement the state transition and decisions related to stack operations.

For the feed-forward networks we use the activation as in \cite{siegelmann1992computational}, namely the saturated linear sigmoid activation function:
\begin{equation}
\sigmoid(x) = \left\{  
\begin{array}{cc}
0 & \quad \text{if  } x < 0, \\
x & \quad \text{if  } 0 \leq x \leq 1, \\
1 & \quad \text{if  } x > 1.
\end{array}
\right.
\end{equation}
Note that, we can easily work with the standard $\mathsf{ReLU}$ activation via $\sigmoid(x) = \mathsf{ReLU}(x) - \mathsf{ReLU}(x-1)$.

Consider two sets $\Phi$ and $\Psi$ (such as the set of states $Q$ and the set of inputs $\Sigma$), and consider the one-hot representations of their elements $\phi \in \Phi$ and $\psi \in \Psi$ as $\vphi \in \bq^{|\Phi|}$ and $\vpsi \in \bq^{|\Psi|}$ respectively. We use $(\phi, \psi) \in \bq^{|\Phi|\times |\Psi|}$ to denote a unique one-hot encoding for each pair of $\phi$ and $\psi$. More specifically, consider the enumerations $\pi_{\Phi} : \Phi \rightarrow \{1, 2, \ldots , |\Phi|\}$ and $\pi_{\Psi} : \Psi \rightarrow \{1, 2, \ldots , |\Psi|\}$. Then given two elements $\phi \in \Phi$ and $\psi \in \Psi$, the vector $(\phi, \psi) \in \bq^{|\Phi|\times |\Psi|}$ will have a $1$ in position $(\pi_{\Psi}(\psi)-1)|\Phi| + \pi_{\Phi}(\phi) $ and a $0$ in every other position. We now prove that given a vector $[\vphi, \vpsi]$ containing concatenation of one-hot representations of the elements $\phi$ and $\psi$, there exists a single-layer feedforward network that can produce the vector $(\phi, \psi)$.

\begin{lemma}\label{lem:map_pair}
	There exists a function $O(\vx): \bq^{|\Phi| + |\Psi|} \rightarrow \bq^{|\Phi||\Psi|}$ of the form $\sigma(\vx \rmW+\vb)$ such that,
	\[
	O([\vphi, \vpsi]) = [(\phi, \psi)]
	\]
\end{lemma}

\begin{proof}
	
	Let $\rmA_i$ for $i \in \{1, \ldots, |\Phi|\}$ denote a matrices of dimensions $|\Psi| \times |\Phi|$ such that $\rmA_i$ has $1$s in its $i$-th row and $0$ everywhere else. For any one-hot vector $\vphi$, note that $\vphi \rmA_i = \bm{1}$ if $i= \pi_{\Phi}(\phi)$ or else it is $\vzero$. Thus, consider the transformation,
	
	\[
	\vt_{(\phi,\psi)} = [\ \vphi+\vpsi\rmA_1, \vphi+\vpsi\rmA_2, \ldots, \vphi+\vpsi\rmA_{|\Psi|} ]
	\]

	Note that, the vector $\vt_{(\phi,\psi)}$ has a value $2$ exactly at the position $(\pi_{\Psi}(\psi)-1)|\Phi| + \pi_{\Phi}(\phi) $ and it is either $0$ or $1$ at the rest of the positions. Hence by making use of bias vectors, it is easy to obtain $[(\phi, \psi)]$,
	\[
	\sigma(\vt_{(\phi,\psi)} - \bm{1}) = [(\phi, \psi)]
	\]
	
	which is what we wanted to show.
	
\end{proof}

As an example, consider two sets $\Phi$ and $\Psi$ such that $|\Phi|=3$ and $|\Psi|=2$. Consider two elements $\phi \in \Phi$ and $\psi \in \Psi$ such that $\pi_{\Phi}(\phi)=3$ and $\pi_{\Psi}(\psi) =2$. Hence, this implies the corresponding one-hot vectors $\vphi=[0, 0, 1]$ and $\vpsi =[0, 1]$. According to the construction above, the weight matrix will be,

\[
\mW = \left[\begin{array}{cccccc}
1&0& 0 &1&0&0\\
0&1& 0&0&1&0\\
0&0& 1 &0&0&1\\
\hline
1&1&1 &0&0 & 0\\
0&0&0 &1&1 &1\\
\end{array}\right],
\]

Thus by construction, the output of the feedforward network $\sigma([\vphi, \vpsi]\rmW - \bm{1}) = [0,0,0,0,0,1]$. A similar technique was employed by \cite{perez2019turing} in their Turing completeness result for Transformer. 

We will describe another technical lemma that we will make use of to implement our transition functions and other such mappings. Consider two sets $\Phi$ and $\Psi$ (such as the set of states $Q$ and the set of inputs $\Sigma$ or set of stack symbols $\Gamma$), and consider the one-hot representations of their elements $\phi \in \Phi$ and $\psi \in \Psi$ as $\vphi \in \bq^{|\Phi|}$ and $\vpsi \in \bq^{|\Psi|}$ respectively. Let $\delta : \Phi \times \Psi \rightarrow \Phi$ be any transition function that takes elements of two sets as input and produces an element of one of the sets. Let $[(\phi, \psi)] \in \bq^{|\Phi|\times |\Psi|}$ denote a unique one-hot encoding for each pair of $\phi$ and $\psi$ as defined earlier. We demonstrate that given  $[(\phi_1, \psi)]$ as input, there exists a single layer feedforward network that produces the vector $\vphi_2$ if $\delta(\phi_1, \psi) = \phi_2$.

\begin{lemma}\label{lem:map_trans}
	There exists a function $O(\vx): \bq^{|\Phi||\Psi|} \rightarrow \bq^{|\Phi|}$ of the form $\sigma(\vx \rmW+\vb)$ such that,
	\[
	O([(\phi_1, \psi)]) = \vphi_2
	\]
\end{lemma}

\begin{proof}
	This can easily implemented using a linear transformation. Consider a matrix $\rmA \in \bq^{|\Phi| |\Psi| \times |\Phi|}$. Given two inputs $\phi_i, \phi_k \in \Phi$ and $\psi_j \in \Psi$  such that $\delta(\phi_i, \psi_j) = \phi_k$, the row $(\pi_{\Psi}(\psi_j)-1)|\Phi| + \pi_{\Phi}(\phi_i) $ of the matrix $\rmA$ will be the one-hot vector corresponding to $\phi_k$, that is, $\rmA_{(\pi_{\Psi}(\psi_j)-1)|\Phi| + \pi_{\Phi}(\phi_i),:} = \vphi_k$.

\end{proof}

Similarly, a mapping $\theta: \Phi \times \Psi \rightarrow \{0,1\}^{n}$ can also be implemented using linear transformation using a transformation matrix $\rmA \in \bq^{|\Phi| |\Psi| \times n}$ where each row of the matrix $\rmA$ will be the corresponding mapping $\{0,1\}^n$ for the pair of $(\phi,\psi)$ corresponding to that row.

\begin{lemma}\label{lem:map_cont}
	There exists a function $O(\vx): \bq^{|\Phi||\Psi|} \rightarrow \{0,1\}^n$ of the form $\sigma(\vx \rmW+\vb)$ such that,
	\[
	O([(\phi, \psi)]) = [\{0,1\}^n]
	\]
\end{lemma}
Proof is similar to proof of Lemma \ref{lem:map_trans}.

\begin{proposition}\label{prop:pda}
	For any Deterministic Pushdown Automaton, there exists an RNN that can simulate it.
\end{proposition}

\begin{proof}
	The construction is straightforward and follows by induction. We will show that at the $t$-th timestep, given that the model has information about the state and a representation of the stack, the model can compute the next state and update the stack representation based on the input. More formally, given a sequence $x_1, x_2, \ldots, x_n \in \Sigma^*$, consider that the hidden state vector at the $t$-th timestep is $\vh_t = [\vq_t, \stk_t]$ where $\vq_t \in \bq^{|\Sigma|}$ is a one-hot encoding of the state vector and $\stk_t \in \bq^{|\Gamma|}$ is a representation of the stack based on the cantor-set like encoding. Then, given an input $x_{t+1} \in \Sigma$, we will show how the network can compute $\vh_{t+1} = [\vq_{t+1}, \stk_{t+1}]$. After reading the whole input, a sequence is accepted if $\vq_n$ is in the set of final states $F$ or else it is rejected.
	
	Our construction will use a 5-layer feed forward network that takes as input the vectors $\vh_{t-1}$ and $\vx_t$ at each timestep and produces the vector $\vh_t$. The vectors $\vh_t \in \bq^{|Q|+|\Gamma|}$ will have two subvectors of size $|Q|$ and $|\Gamma|$ containing a one-hot representation of the state of the underlying automaton and a representation of the stack encoded in the Cantor-set representation respectively. For each input symbol $x \in \Sigma$, its corresponding input vector $\vx \in \bq^{|\Sigma|}$ will be a one-hot vector. If the underlying stack takes empty string as input at particular step, the RNN will take a special symbol as input which also have a unique one-hot representation similar to other input symbols.
		
	As opposed to the construction of Siegelmann and Sontag \shortcite{siegelmann1992computational}, which only takes $0$s and $1$s as input and use a scalar to encode a stack of $0$s and $1$s, we will encode one-hot representation of stack symbols in vectors of size $|\Gamma|$. The push and pop operations will always be in the form of one-hot vectors and this will ensure that retrieving the top element provides a one-hot encoding of the stack symbol.
	
	\textbf{Details of the construction.} At each timestep, the model will receive the hidden state vector $\vh_{t-1}= [\vq_{t-1}, \stk_{t-1}]$ and the input vector $\vx_t$ as input. At timestep $0$, the hidden state vector will be $\vh_0 = [\vq_0, \stk_0]$ containing the one-hot representation of the initial state and stack encoding containing $Z_0$. For instance consider $|Q| =3$ and $|\Gamma|=4$. If the one-hot encoding of $q_0$ is $\vq_0= [1,0,0]$ and one-hot encoding of $Z_0$ is $\vz_0= [1,0,0,0]$, then $\stk_0 = \frac{1}{4}\vzero + \frac{1}{2}[1,0,0,0]+ \bm{\frac{1}{4}} = [3/4,1/4,1/4,1/4]$. That is, the vector $\vz_0$ is pushed to the empty stack using the Cantor-set encoding method described in section \ref{subsec:cantor}. Hence, the vector $\vh_0= [1,0,0,3/4,1/4,1/4,1/4]$.
	
	The first layer of the feedforward network $\sigma(\rmW_h\vh_{t-1} + \rmW_x\vx_t + \vb )$ will produce the vector $\vh_{t-1}^{(1)}= [(q_{t-1}, x_t ), \topv, \stk_{t-1}]$, where $\topv \in \bq^{|\Gamma|}$ denotes a one-hot vector representation of the symbol at the top of our stack representation and the subvector $(q_{t-1}, x_t ) \in \bq^{|Q|\times |\Sigma|}$ is a unique one-hot vector for each pair of state $q$ and input $x$. Thus, the vector $\vh_{t-1}^{(1)}$ is of dimension $|Q|.|\Sigma| + 2|\Gamma|$. The vector $(q_{t-1}, x_t )$ can be obtained by using Lemma \ref{lem:map_pair} where $\Phi = Q$ and $\Psi = \Sigma$. The vector corresponding to the symbol at the top of the stack can be easily obtained using the top operation ($\sigma(\bm{4}\stk_{t-1} - \bm{2})$) defined in section \ref{subsec:cantor}.
	
	In the second layer, we will use Lemma \ref{lem:map_pair} again to obtain a unique one-hot vector for each combination of the state, input and stack symbol. The output of the second layer of the feedforward network will be of the form $\vh_{t-1}^{(2)}= [(q_{t-1}, x_t, \topv ), \topv, \stk_{t-1}]$, where the subvector $(q_{t-1}, x_t, \topv)$ is a unique one-hot encoding for each combination of the state $q \in Q$, input $x \in \Sigma$ and a stack symbol $\tau \in \Gamma$. Hence, the vector $\vh_{t-1}^{(2)}$ will be of the dimension $|Q|.|\Sigma|.|\Gamma| + 2|\Gamma|$. Since we already had the vector $(q_{t-1}, x_t)$, and the vector $\topv$, the vector $(q_{t-1}, x_t, \topv)$ can be obtained using Lemma \ref{lem:map_pair} by considering $\Phi =Q \times \Sigma$ and $\Psi= \Gamma$. The primary idea is that the vector $(q_{t-1}, x_t, \topv)$ provides us with all the necessary information required to implement the further steps and produce $\vq_t$ and $\stk_t$.
	
	We can use the vector $(q_{t-1}, x_t, \topv)$ to directly map to $\vq_t$ using a simple linear transformation. To obtain $\stk_t$, we will produce all three candidate stack representations corresponding to push, pop and no-operation. That is, from $(q_{t-1}, x_t, \topv)$ we will obtain $\stk_{t}^{push}$, $\stk_{t}^{pop}$ and $\stk_{t}^{no-op}$. Along with that, using linear transformation we will obtain three control signals $\vc_{push}, \vc_{pop} \text{ and } \vc_{no-op}$. To obtain the final stack representation $\stk_t$, we will implement the following operation
	\begin{equation}
	 \stk_t = \vc_{push}.\stk_{t}^{push} + \vc_{pop}.\stk_{t}^{pop} + \vc_{no-op}.\stk_{t}^{no-op}.
	\end{equation}

	 We will implement the above steps using an additional three layers of feedforward network and thus we will obtain $\vh_t = [\vq_t, \stk_t]$.
	 
	 The third layer of the feedforward network will produce the vector 
	 \[
	 	 \vh_{t-1}^{(3)}= [\vq_{t}, \stk_t, \stk_{t}^{pop}, \stk_{t}^{no-op}, \psh,  \vc_{push}, \vc_{pop}, \vc_{no-op} ]
	 \]
	 
	 The vector $\vq_t$ in $\vh_{t-1}^{(3)}$ can be obtained using Lemma \ref{lem:map_trans} given the vector $(q_{t-1}, x_t, \topv)$ in $\vh_{t-1}^{(2)}$. The vector $\stk_{t}^{pop}$ via the Cantor set encoding method using the transformation $\bm{4}\stk_{t} - \bm{2}\topv - \bm{1}$ over $\vh_{t-1}^{(2)}$. The vector $\stk_{t}^{no-op}$ can be obtained using the Identity Transformation. The vector $\psh$ can be obtained using Lemma \ref{lem:map_trans}. If for a given transition the stack operation is not a push operation then $\psh = \vzero$. The vector $\vc_{push} = \bm{1}$ if the current stack operation is Push and it is $\vzero$ otherwise. Similarly, the vectors $\vc_{pop}$ and $\vc_{no-op}$ are $\bm{1}$ if the current operation is Pop or No-operation respectively and are $\vzero$ otherwise. The vectors $\vc_{push}, \vc_{pop},\text{ and } \vc_{no-op}$ can be obtained using Lemma \ref{lem:map_cont}. During any timestep, only one of them is $\bm{1}$ vector and rest of them are zero vectors. The vectors $\vc_{op}$s can be seen as control signals. That is, the candidate stack representation will be used if its corresponding control signal is $\bm{1}$ or else the candidate stack representation will be transformed to zero vector.
	 
	 In the fourth layer, the feedforward network will produce the following vector,
	 \[
	 \vh_{t-1}^{(4)}= [\vq_{t},  \stk_{t}^{pop}. \vc_{pop}, \stk_{t}^{no-op}. \vc_{no-op}, \stk_{t}^{push}. \vc_{push}]
	 \]
	 
	 For any operation denoted by $op \in \{push, pop, no-op\}$, the vector 
	 \begin{equation}\label{eq:mul_op}
	 \stk_t^{op}. \vc_{op} = \left\{  
	 \begin{array}{cc}
	 \stk_t^{op} & \quad \text{if  } \vc_{op} = \bm{1}, \\
	 \vzero& \quad \text{if  } \vc_{op} = \vzero.
	 \end{array}
	 \right.
	 \end{equation}
	 Since we are using first order RNNs and hence multiplicative operations are not directly possible. To implement the operation in equation \ref{eq:mul_op} via linear transformations with saturation linear sigmoid activation, first note that $0 \leq \stk_t \leq 1$ and thus $\sigma(\stk_t - \bm{1}) = \vzero$. Using that we can implement the operation in equation \ref{eq:mul_op} by, 
	 
	 \begin{equation}\label{eq:lin_op}
	 \sigma(\stk_t^{op} + \vc_{op} - \bm{1})= \left\{  
	 \begin{array}{cc}
	 \stk_t^{op} & \quad \text{if  } \vc_{op} = \bm{1}, \\
	 \vzero& \quad \text{if  } \vc_{op} = \vzero.
	 \end{array}
	 \right.
	 \end{equation}
	 
	 In the fourth layer of the feedforward network we obtain the candidate stack representation for Push operation via linear operations as described in section \ref{subsec:cantor} along with adding $\vc_{push}$ and adding $-\bm{1}$ via bias vectors. For the candidate stack representations of Pop and No-op operations, we simply add their control vectors and subtract by $\bm{1}$ via bias vectors. In the vector $\vh_{t-1}^{(4)}$, only one of the stack representation will have nonzero values and the other representations will be zero vectors depending on the control signals $\vc_{push}, \vc_{pop},\text{ and } \vc_{no-op}$.
	 
	 The fifth layer will simply sum the three candidate stack vectors along with their control signals to obtain,
	 
	  \[
	 \vh_{t-1}^{(5)}= [\vq_{t},  \stk_{t}^{pop}. \vc_{pop} + \stk_{t}^{no-op}. \vc_{no-op} + \stk_{t}^{push}. \vc_{push}]
	 \]
	 which is equal to,
	 \[
	 [\vq_{t}, \stk_t] = \vh_t
	 \]
	 
	 which is what we wanted to show.

\end{proof}

The above construction is a direct simulation of Deterministic Pushdown Automaton via RNNs in real-time. The construction of \newcite{siegelmann1992computational} first takes all the inputs and then takes further processing time. However, in practice, RNNs process the inputs and produce outputs in real-time. 

Note that, the dependence on precision is primarily determined by the depth of the stack. That is, as the number of elements in the stack keep increasing, the stack representation gets exponentially smaller due to the Cantor-set encoding scheme. If for a set of input strings, the depths are bounded, then for some finite precision, an RNN based on the above construction will be able to recognize strings of arbitrary lengths.

\section{Experimental Details}\label{sec:exp_details}

\begin{table*}[t]
\small{\centering

\begin{tabular}{P{6em}P{3em}P{4em}P{4em}P{4em}P{4em}P{4em}P{3em}P{3em}P{3em}}
\toprule
&\multicolumn{3}{c}{\textbf{Training Data}} &\multicolumn{6}{c}{\textbf{Test Data}}\\
\cmidrule(lr){2-4}\cmidrule(lr){5-10}
Dataset & Size  & Length Range & Depth Range & Size per Bin & Bin-2 Length Range & Bin-2 Depth Range & Length Increments & Depth Increments & Number of Bins \\
\midrule
Unbounded Length and Depth & 10000& [2, 50] & Unrestricted & 1000 & [52, 100] & Unrestricted & 50 & NA & 2\\
\midrule
Bounded Depth & 10000 & [2, 50] & [1, 10] & 1000 & [52, 100] & [1, 10] & 50 & 0 & 3\\
\midrule
Bounded Length & 10000& [2, 100] & [1, 15] & 1000 & [2, 100] & [16, 16] & 0 & 1 & 6\\
\bottomrule
\end{tabular}
\caption{\label{tab:stats} Statistics of different datasets used in the experiments. Note that the distribution of the first bin is always defined by the training set, hence for test set we report the statistics from the second bin. The depth and length bounds for the other bins can be obtained by considering the lower bound of a bin $i$ as one plus the upper bound of the previous bin $i-1$ and obtaining the upper bounds by adding the Length and Depth Increments to the previous bin's upper bounds.}

}
\end{table*}

\begin{table*}[t]
\normalsize{\centering
\begin{tabular}{P{9em}P{15em}}
\toprule
\textbf{Hyperparameter} & \textbf{Bounds}\\
\midrule
Hidden Size & [4, 256]\\
Heads & [1, 4] \\
Number of Layers & [1, 2]\\
Learning Rate & [1e-2, 1e-3]\\
Position Encoding & [True, False]\\
\bottomrule
\end{tabular}
\caption{\label{tab:hyps} Different hyperparameters and the minimum and maximum values considered for each of them. Note that certain parameters like Heads and Position Encodings are only relevant for Transformer based models and not for LSTMs. Note that, for Dyck-2 we also found them to generalize to Bin-1A with bounds mentioned in \newcite{suzgun2019lstm} (such as with hidden size 8).}
}
\end{table*}

We run our experiments on three Context Free Languages namely \dyck{2}, \dyck{3} and \dyck{4} for LSTM and Transformer based models. Three separate datasets were generated for each language, to run the ablations, details of which are given in Table \ref{tab:stats}. For each of these experiments we do extensive hyperparameter tuning before reporting the final results. Table \ref{tab:hyps} provides different hyperparameters considered in our experiments and their bounds. All in all, this resulted in about 56 different settings for each dataset for RNNs and about 144 settings for Transformers. While reporting the final scores we take an average of the accuracies corresponding the top-5 hyperparameter settings.

All of our models were trained using RMSProp Optimizer with a smoothing constant $\alpha$ of 0.99. For each language and its corresponding datasets, we use a batch size of 32 and train for 100 epochs. In case an accuracy of 0.99 is acheived for all of the bins before completing 100 epochs we stop the training process at that point. We run all of our experiments on 4 Nvidia Tesla P100 GPUs each containing 16GB memory.

\textbf{Probing Details} For designing a probe for extracting the depth of the underlying stack of \dyck{2} substrings from the cell states of pretrained LSTMs, we used a single hidden layer Feed-Forward network. The hidden size of the network was kept $32$ and it was trained using Adam Optimizer \cite{kingma2014adam} with a batch size of 200. The accuracy on the validation set was computed by only considering the predictions to be correct for a sequence if it predicted the correct depth at every step of that sequence. In the second set of experiments that aimed to predict the elements of the stack, we trained the LSTM model on the NCP task along with an auxillary loss for predicting the top-10 elements of the stack. For the computation of the auxillary loss, we added 10 parallel linear layers on the top of the LSTM's output with $i$-th linear layer tasked to predict if the $i$-th element of the stack was i) a round opening bracket or ii) a square opening bracket or iii) If no element was present at that position.  For each of these 10 linear layers we compute Cross Entropy Loss which are then averaged to obtain the auxillary loss. The final loss is computed as:
\begin{equation}
	L = L_{NCP} + \lambda L_{aux}
\end{equation}
where $L_{NCP}$ is the loss obtained from the next character prediction task and $L_{aux}$ is the auxillary loss just described and we use $\lambda = \frac{1}{20}$ in our experiments. The stack extraction auxillary task is evaluated by computing Accuracy and Recall metrics for each stack element. The accuracy for the $i$-th element is computed by considering if the model can predict the $i$-th element correctly at each step of a sequence. Since there will be a fewer cases for smaller lengths containing elements at higher depths, we also report Recall for each $i$-th stack element, where we only consider if the model can correctly predict the sequences containing at least one occurrence of depth $i$.

\section{Robustness Experiments}

\begin{table*}[t]
	\scriptsize{\centering
		\begin{tabular}{P{6em}p{7em}P{8em}P{8em}P{8em}}
			\toprule
			\textbf{Language} & \textbf{Model} &
			\textbf{Validation Set 1 $p = 0.5, q = 0.25$}&
			\textbf{Validation Set 2 $p = 0.4, q = 0.35$}&
			\textbf{Validation Set 3 $p = 0.6, q = 0.15$} \\
			
			\midrule
			\multirow{2}{*}{\textbf{\dyck{2}}} & \textbf{LSTM} & 99.5 & 99.6 & 99.1\\
			&\textbf{Transformer} & 95.1 & 94.7 & 94.4\\
			\midrule
			\multirow{2}{*}{\textbf{\dyck{3}}}&\textbf{LSTM} & 97.3 & 98.8 & 96.5\\
			&\textbf{Transformer} & 87.7 & 87.8 & 87.3\\
			\midrule
			\multirow{2}{*}{\textbf{\dyck{4}}}&\textbf{LSTM} & 97.8 & 96.7 & 94.9\\
			&\textbf{Transformer} & 92.7 & 89.8 & 90.9\\
			\bottomrule
		\end{tabular}
		\caption{\label{tab:rob_results} The performance of neural models on considered Dyck languages for data generated from three different distributions. Validation Set 1 was constructed from the same distribution used to generate the training data, while the other two were generated from different distributions. All the validation sets had strings with the lengths in the interval [2, 50] and there was no restriction kept on the depth of these strings.}
	}
\end{table*}

To ensure that our results didn't overfit on the training distribution we did some robustness experiments to check efficacy of the considered neural models. As a reminder, the PCFG for \dyck{n} languages is given by the following derivation rules: 
\begin{align}
S &\rightarrow (_iS)_i, \text{  with probability $p$}\\
S &\rightarrow SS, \text{  with probabililty $q$}\\
S &\rightarrow \epsilon, \text{ with probbaility $1 - (p+q)$}\\
\end{align}

For the experiments described in the main paper we used $p = 0.5$ and $q = 0.25$. To check the generalization ability of our models we checked whether a model trained with the strings generated using these values can generalize on Dyck words obtained from a different distribution. Table \ref{tab:rob_results} shows the accuracies obtained by LSTMs and Transformers on data generated from different distributions. It can be observed from the results that the performance of the models for all languages remain more or less the same across different distributions.

A model trained on the Next Character Prediction task can also be used to generate strings of the language it was trained on by starting from an empty string and then exhaustively iterating over all possible valid characters predicted by the model. We used this idea to check if a pretrained LSTM model can indeed generate all possible \dyck{2} strings upto a certain length (since the number of possible strings will grow exponentially with increasing lengths). For a maximum length of $10$, there exists a total of $1619$ valid \dyck{2} strings. When we used a pretrained model to exhaustively generate the valid strings, we observed that it produced exactly those $1619$ strings, no more and no less.

\section{Additional Results}

The depth generalization results for the neural models on \dyck{3} and \dyck{4}, are given in Figure \ref{fig:depth_gen_dycks34}. Similar to \dyck{2}, here also we see a gradual drop in performance as we move to the higher depths but Transformers perform relatively better than LSTMs.

\begin{figure}[h!]
\centering
\begin{subfigure}[b]{0.45\textwidth}
\centering
\includegraphics[width=.9\textwidth]{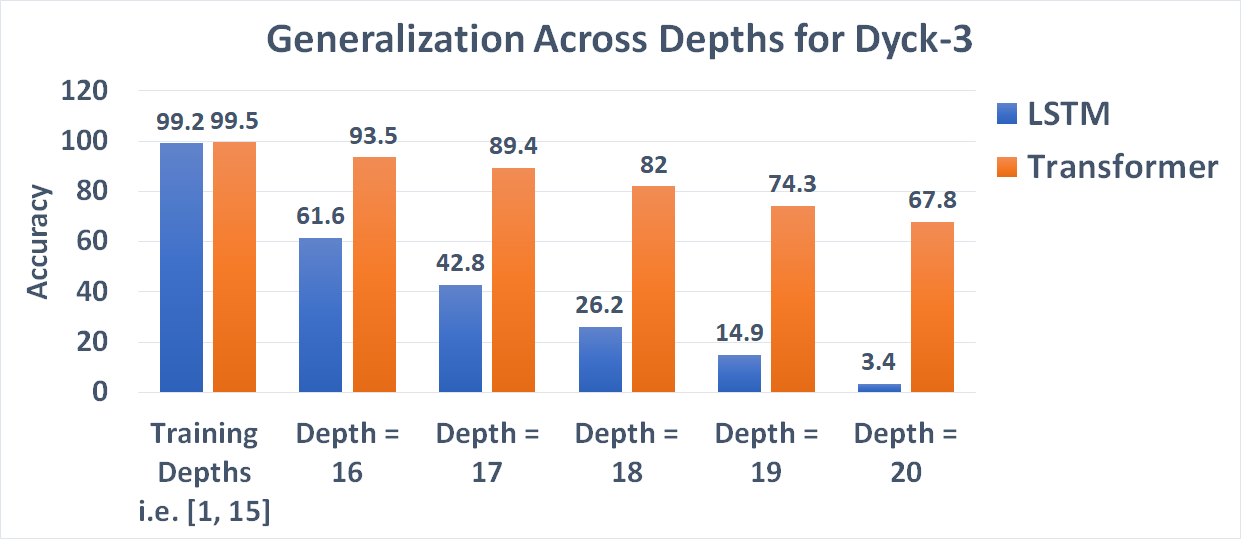}
\caption{\label{fig:depth_gen_dyck3}}
\end{subfigure}%
\begin{subfigure}[b]{0.45\textwidth}
\centering
\includegraphics[width=.9\textwidth]{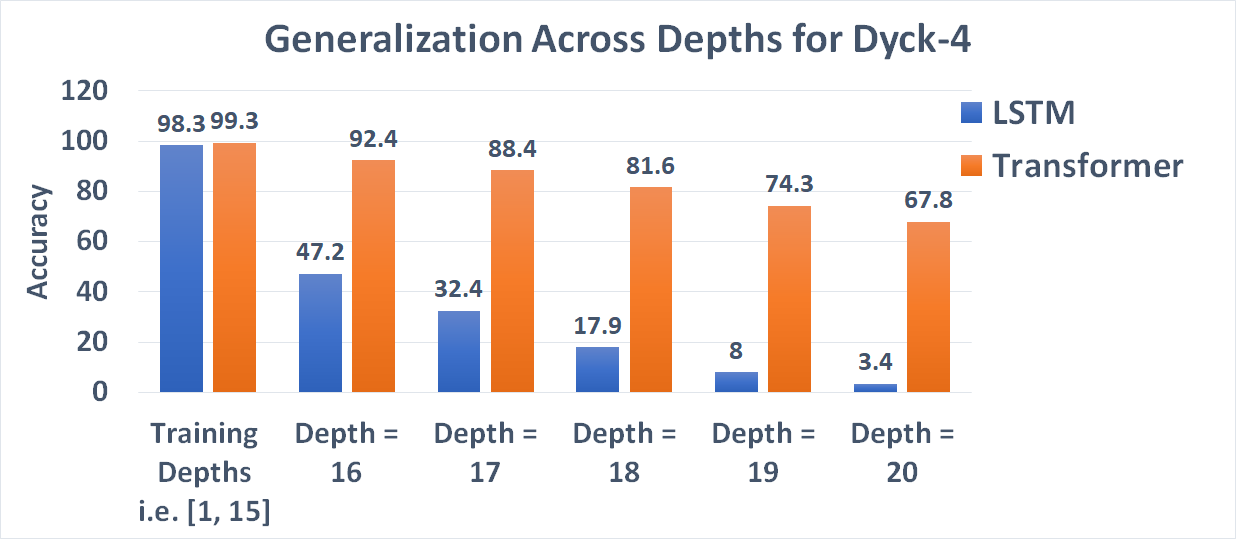}
\caption{\label{fig:depth_gen_dyck4}}
\end{subfigure}
\caption{\label{fig:depth_gen_dycks34} Generalization of LSTMs and Transformers on higher depths for (a) \dyck{3} and (b) \dyck{4}. The lengths of strings in the training set and all validation sets were fixed to lie between 2 to 100.}
\end{figure}

\end{document}